\documentclass[11pt]{article}

\usepackage[final]{acl}

\usepackage{times}
\usepackage{latexsym}
\usepackage[T1]{fontenc}
\usepackage[utf8]{inputenc}
\usepackage{microtype}
\usepackage{inconsolata}
\usepackage{graphicx}
\usepackage{booktabs}
\usepackage{amsmath}
\usepackage{amssymb}
\usepackage{amsthm}

\newtheorem{proposition}{Proposition}

\title{The Information Shadow: Measuring Structural Limits on \\ What Language Models Can Learn}

\author{Priyansh Srivastava \\
  Sirena Ai \\
  \texttt{priyansh@sirenatech.com} \\\And
  Romit Chatterjee \\
  Independent Researcher \\
  \texttt{chatterjeeromit86@gmail.com}}

\begin{document}
\maketitle

\begin{abstract}
Some limits on what language models know are not gaps in data coverage but structural properties of learning from text. We introduce the \emph{information shadow}: the region of phenomena that a text-trained learner cannot acquire regardless of scale, comprising (I) structures language cannot express, (II) functions that are statistically non-identifiable from the training distribution, and (III) functions that are representable but unreachable by gradient-based training. We give each type a probe that is decisive because the premise of the shadow is, in that setting, \emph{provable}. For Type~I, Language Compression Residuals compare a \emph{text learner}, which sees only a lossy text-like encoding of the signal, against a \emph{full-signal learner}, which sees the underlying signal directly: the text learner sits at a \emph{computable} expressibility ceiling while the full-signal learner pulls away by a gap that stays flat across $300\times$ more data, so the deficit is a property of the channel, not of training. For Type~II, the Counterfactual Distinction Test trains models on data exactly consistent with two incompatible rules; across a provable string task and a language-like agreement task, behavior on counterfactuals is set by the model's inductive bias, while $5\%$ disambiguating data steers the learned rule \emph{bidirectionally} to either target ($r=\pm1.0$, $p<10^{-10}$). For Type~III, Basin Escape Mapping exhibits a function that is representable at $100\%$ (by hand-construction) yet reached $0\%$ of the time by standard training and instantly from a nearby initialization, with width scaling no help ($p=1.6\times10^{-14}$). Each effect is isolated by a control that rules out a capacity or modality artifact. We release the probe suite and discuss implications for benchmark design, capability auditing, and shadow-aware uncertainty.
\end{abstract}

\section{Introduction}
\label{sec:intro}

Ask a language model (LM) for a limerick and it obliges; ask it for the path of every raindrop in a storm ten years ago and it cannot help. The second failure is unremarkable—missing data. The interesting question is whether there exist limits that persist under \emph{arbitrarily} more data, compute, and architectural refinement.

We argue the answer is yes, and that these limits are heterogeneous: they arise from at least three distinct sources with different signatures and different remedies. We call their union the \emph{information shadow}. A phenomenon lies in the shadow if (\textbf{Type~I}) its structure cannot be adequately carried by symbolic language; (\textbf{Type~II}) multiple incompatible mechanisms induce the same distribution over observable text, so no amount of that text identifies the true one; or (\textbf{Type~III}) the function is representable by the architecture but standard training dynamics do not reach it.

Distinguishing these types matters practically, because each resists a different remedy. A \textbf{Type~I} failure cannot be fixed by more text: it is a property of the channel (language) itself, so a learner with access to a richer, non-linguistic signal pulls ahead by a margin that more text never closes. A \textbf{Type~II} failure cannot be fixed by scaling: it is a property of the data-generating process, not of capacity, so the model falls back on its \emph{inductive bias}---the built-in preferences (from architecture and optimizer) that decide which of several data-consistent rules it adopts. A \textbf{Type~III} failure cannot be fixed by more of the \emph{same} data or more parameters: the function is representable, but the optimization path never reaches it. Conflating these leads to misdirected mitigation and to benchmarks that overstate understanding \citep{geirhos2020shortcut,mccoy-etal-2019-right}.

This paper contributes:
\begin{enumerate}
    \item \textbf{A taxonomy} of three structural limit types with a falsifiable, measurable signature for each (\S\ref{sec:taxonomy}).
    \item \textbf{Language Compression Residuals} for Type~I (\S\ref{sec:lcr}): a text learner is shown to sit at a \emph{computable} expressibility ceiling while a full-signal learner with access to the same information uncompressed pulls away by a gap that does not close with scale, with a coarse control isolating inexpressibility from a generic modality gap.
    \item \textbf{The Counterfactual Distinction Test} for Type~II (\S\ref{sec:cdt}--\ref{sec:results}), built on \emph{provably} non-identifiable training distributions, with controls that steer the learned rule bidirectionally on $5\%$ of the data, validated on a string task and a language-like agreement task.
    \item \textbf{Basin Escape Mapping} for Type~III (\S\ref{sec:bem}), pairing a constructive proof that the target is representable with measured reachability rates under matched interventions (initialization, curriculum, width).
    \item \textbf{Effect sizes, multiplicity-corrected tests, and dose--response curves} throughout, and a released, resumable probe suite (\S\ref{sec:discussion}).
\end{enumerate}

Our experiments use small models on synthetic tasks by design: a separation result requires settings where each shadow's premise is a theorem, not an assumption. Section~\ref{sec:related} situates this against observational studies of the same phenomena in large pretrained models.

\section{Related Work}
\label{sec:related}

\paragraph{Expressivity and the limits of form (Type~I).}
Type~I connects to formal expressivity bounds: self-attention cannot model parity or unbounded hierarchy without growing resources \citep{hahn2020theoretical,raghu2017expressive} (though \citealp{chiang2022overcoming} show finite-precision constructions can recognize parity, so the bound is asymptotic), log-precision transformers lie in $\mathrm{TC}^0$ \citep{merrill2023parallelism}, and surveys map architecture to the Chomsky hierarchy \citep{deletang2023chomsky,strobl2024formal}. The simplicity bias of the parameter--function map \citep{valleperez2019deep} and arguments that meaning cannot be recovered from form \citep{bender-koller-2020-climbing,merrill2021provable} motivate a residual that text cannot close. Language Compression Residuals make this measurable: we compare a text learner against a full-signal learner with a \emph{computable} Bayes ceiling, isolating inexpressibility from a generic modality gap via a coarse control. Our taxonomy reorganizes the classical expressivity/trainability/generalization trichotomy around \emph{what no amount of text can fix}, attaching a falsifiable probe and a control to each axis.

\paragraph{Underdetermination across random seeds (Type~II).}
\citet{mccoy-etal-2020-berts} fine-tuned 100 BERT instances on MNLI: in-distribution accuracy was nearly constant (83.6--84.8\%) while HANS subject--object accuracy ranged from 0.0\% to 66.2\%. \citet{damour2020underspecification} call this \emph{underspecification}; the space of equally-good models is formalized by the Rashomon set \citep{semenova2022existence} and predictive multiplicity \citep{marx2020multiplicity}, and learned representations are identifiable only up to linear transformation \citep{roeder2021linear}. These characterize the \emph{set} of solutions but cannot rule out that a subtle cue identifies the rule, nor steer a learner between solutions. Our CDT removes the confound—non-identifiability holds \emph{by construction} (Prop.~\ref{prop:nonid})—and turns underdetermination into a controllable dial.

\paragraph{Ambiguous training, disambiguating evaluation (Type~II).}
MSGS \citep{warstadt-etal-2020-learning} trains on data consistent with both a linguistic and a surface rule, scores a Matthews-correlation preference in $[-1,1]$, and shows small amounts of disambiguating data (``inoculation''; \citealp{liu-etal-2019-inoculation}) shift the preference. \citet{mccoy-etal-2020-syntax} find the adopted hierarchical-vs-linear rule tracks architecture. On the theory side, \citet{abbe2023curriculum} prove that altering the input distribution provably accelerates parity learning. We adopt MSGS's preference-score convention and extend this line in three ways: provable (not merely empirical) ambiguity; \emph{bidirectional} steering, where the same $5\%$ budget pins \emph{either} rule; and a dose--response curve over the disambiguating fraction. Unlike \citet{abbe2023curriculum}, who establish a one-directional speed-up, we measure the full identifiability response and its inductive-bias dependence behaviorally.

\paragraph{Reachability and the dynamics of SGD (Type~III).}
That representable functions can be unreachable is well studied for parities: \citet{shalev-shwartz2017failures} show gradient methods fail on random parities, \citet{barak2022hidden} show SGD learns $k$-sparse parities near the computational limit with abrupt ``hidden progress,'' and the merged-staircase/leap-complexity results \citep{abbe2022merged,abbe2023leap} characterize which sparse functions two-layer networks reach. Good subnetworks need not be discoverable \citep{frankle2019lottery}, and grokking \citep{power2022grokking} ties reachability to training time and data fraction. Basin Escape Mapping does not re-derive these thresholds; it operationalizes them—pairing a constructive representability proof with reachability rates under matched interventions (init, curriculum, width)—so the representable/reachable gap becomes a reportable quantity, with width scaling shown to be no remedy.

\paragraph{Identifiability theory (Type~II).}
Type~II is a learning-theoretic identifiability problem \citep{valiant1984theory}: distinct hypotheses inducing identical observable distributions cannot be distinguished. Nonlinear ICA is unidentifiable without auxiliary signal \citep{hyvarinen1999nonlinear}, and the causal hierarchy implies counterfactual quantities are generically underdetermined by observational data \citep{pearl2009causality}. Our contribution is not these theorems but a behavioral protocol that turns them into a probe for neural learners.

\section{The Information Shadow}
\label{sec:taxonomy}

Table~\ref{tab:taxonomy} summarizes the three types; each has a distinct \emph{signature}—an observable pattern that distinguishes it from ordinary data scarcity—and a distinct mitigation class.

\begin{table}[t]
\centering
\small
\begin{tabular}{@{}p{0.5cm}p{2.5cm}p{3.6cm}@{}}
\toprule
 & \textbf{Cannot be fixed by} & \textbf{Probe \& signature} \\
\midrule
I & more text & LCR: text plateaus at a computable ceiling; full-signal pulls away \\
\addlinespace
II & more scale & CDT: bias sets the rule; $5\%$ data re-pins it \\
\addlinespace
III & more parameters & BEM: representable but vanilla SGD reaches it $0\%$; near-init rescues \\
\bottomrule
\end{tabular}
\caption{The three information-shadow types. Each resists a different ``just add more'' remedy and is detected by a dedicated probe with a falsifiable signature.}
\label{tab:taxonomy}
\end{table}

\paragraph{Why scaling does not dissolve the shadow.}
Type~I limits are properties of the channel (language), not the learner: if the signal needed for a task is destroyed when reality is compressed into text, no amount of additional text restores it. Type~II limits are properties of the data-generating process: if two mechanisms agree on the support of the training distribution, more samples from that distribution add no information about which is true. Type~III limits can in principle worsen with scale, as richer function classes induce more fragmented optimization landscapes. Scaling shifts \emph{which} phenomena sit in the shadow; it does not empty it.

\section{Type I: Language Compression Residuals}
\label{sec:lcr}

Type~I holds that some structure cannot be carried by text at all, so a full-signal learner beats a text-only learner by a margin that \emph{never closes with scale}---because the channel, not the learner, is the bottleneck. Results appear in \S\ref{sec:results-lcr}.

\paragraph{Design.}
A latent $z\in[0,1]$ generates the target. The \emph{text learner} sees only ``text''---a lossy $B$-bin quantization of $z$, i.e.\ a coarse ``word'' that discards everything below bin resolution---whereas the \emph{full-signal learner} sees $z$ itself, the same information uncompressed. The two learners are otherwise identical (same architecture, optimizer, and budget), so any gap between them is due to the compression of the input, not the learner. The \textbf{fine} target $\mathbf{1}[z_1{>}z_2]$ on near-tied pairs depends on sub-bin structure that quantization destroys; its Bayes-optimal text accuracy is \emph{computable} (cell-wise max-posterior over bin pairs) and plotted as a ceiling. The \textbf{coarse} control $\mathbf{1}[b(z_1){\ge}b(z_2)]$ is a deterministic function of the text---if the method merely detected a modality gap it would fire here too. We train text-only and full-signal MLPs at budgets $300$ to $100$k, $N{=}8$ seeds.

\section{Type II: The Counterfactual Distinction Test}
\label{sec:cdt}

The CDT operationalizes Type~II. Its logic (Figure~\ref{fig:cdt}): (1) construct a training distribution that is \emph{exactly} consistent with two incompatible rules---in our string task, \textsc{copy} and \textsc{reverse}; (2) train many models that fit it (near-)perfectly; (3) probe on counterfactual inputs where the rules diverge—any systematic behavior there cannot have come from the data; (4) run \emph{identifiable controls} in which a small fraction of disambiguating examples is injected, labeled by a planted rule—if behavior then converges to the planted rule (in either direction), the shadow behavior was underdetermined by data, not forced by capacity or architecture. We use two terms throughout: the \textbf{shadow} condition trains on the ambiguous (non-identifiable) data alone, while a \textbf{control} condition adds a small, labeled fraction that makes one rule identifiable, serving as a reference for what the model does when the data \emph{is} informative.

\begin{figure}[t]
  \centering
  \includegraphics[width=\columnwidth]{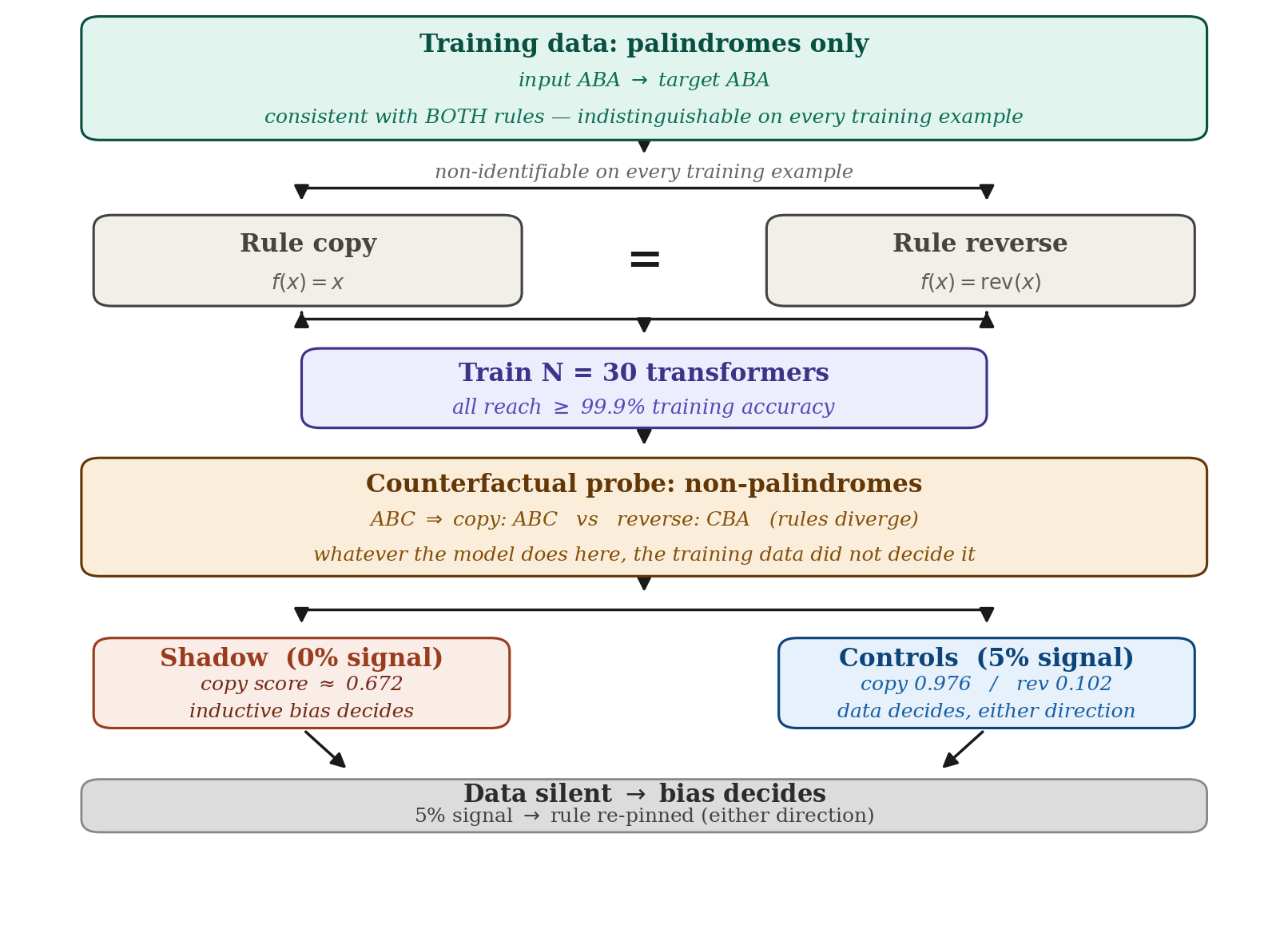}
  \caption{The Counterfactual Distinction Test on the palindrome task. Palindrome-only data is consistent with both \textsc{copy} and \textsc{reverse}; the rules diverge only on the non-palindrome probe. ``$N{=}30$ seeds'' denotes 30 models trained from different random initializations on the same fixed data. Controls inject $5\%$ disambiguating examples and re-pin the rule in either direction.}
  \label{fig:cdt}
\end{figure}

\subsection{Instantiation 1: a provable string task}
\label{sec:palindrome}

Inputs are strings over a $k$-symbol alphabet; the target is a string of the same length. Two candidate rules: $f_{\textsc{copy}}(x)=x$ and $f_{\textsc{rev}}(x)=\mathrm{reverse}(x)$. The training set contains only \emph{palindromes}, labeled with $y=x$.

\begin{proposition}
\label{prop:nonid}
Let $\mathcal{D}$ be any distribution supported on palindromes with targets $y=x$. Then $f_{\textsc{copy}}$ and $f_{\textsc{rev}}$ have identical behavior on every example in the support of $\mathcal{D}$, and hence identical risk under any loss. No learner observing samples from $\mathcal{D}$ can distinguish them.
\end{proposition}
\begin{proof}
If $x$ is a palindrome, $\mathrm{reverse}(x)=x$, so $f_{\textsc{rev}}(x)=f_{\textsc{copy}}(x)$ pointwise on the support. \end{proof}

At test time we present \emph{non-palindromes}, on which the two rules disagree at every position $i$ with $x_i \neq x_{L-1-i}$ (\emph{disambiguating positions}). Whatever a trained model outputs there was not determined by its training data.

\paragraph{Conditions.}
\textbf{Shadow}: palindromes only (non-identifiable). \textbf{Control-\textsc{copy}} / \textbf{Control-\textsc{rev}}: identical except a fraction $\rho$ of training examples are non-palindromes labeled by the planted rule, making it identifiable; $\rho{=}0.05$ in the main conditions, with a dose--response sweep $\rho \in \{0.01, 0.02, 0.05, 0.10\}$. \textbf{Architecture sweep}: the shadow condition repeated with rotary instead of learned absolute position encodings, testing whether the data-silent choice tracks inductive bias.

\subsection{Instantiation 2: a language-like task}
\label{sec:bridge}

To show the construction is not an artifact of string manipulation, we build the same logic into \emph{subject--verb agreement}, a phenomenon with a long history in LM analysis \citep{linzen2016assessing}. Inputs are noun sequences ending in a verb cue; the model outputs the verb's number. Here the two incompatible rules are \textsc{hierarchical} (agree with the subject, the first noun) and \textsc{linear} (agree with the most recent noun); both fit ordinary training data. In training, all nouns in a sentence share one number, so the rules coincide—non-identifiable. The counterfactual probe inserts an opposite-number \emph{attractor} between subject and cue, where the rules diverge; this mirrors agreement-attraction designs \citep{linzen2016assessing,mccoy-etal-2020-syntax}. Controls plant either rule via $\rho{=}0.05$ disambiguating sentences.

\subsection{Metrics}
\label{sec:metrics}

At each disambiguating position we classify the model's output token as matching the \textsc{copy} target, the \textsc{rev} target, or neither. The \textbf{copy score} $c \in [0,1]$ is the fraction of \textsc{copy} matches among rule-consistent positions ($c{\to}1$: \textsc{copy}; $c{\to}0$: \textsc{rev}; $c{\approx}0.5$: undetermined). For comparability with MSGS we also report the \textbf{preference score} $s = 2c-1 \in [-1,1]$ \citep{warstadt-etal-2020-learning}. \textbf{Rule consistency} is the fraction of disambiguating positions where the output matches \emph{either} rule, verifying the model learned the task rather than emitting noise. The bridge task uses the analogous \textbf{subject score} (1 = \textsc{hierarchical}, 0 = \textsc{linear}).

Across $N$ seeds per condition we report means with bootstrap 95\% CIs, two-sided Mann--Whitney tests between shadow and each control with rank-biserial effect sizes, a Holm--Bonferroni multiple-comparison correction, Levene's test for variance, and an across-seed per-position disagreement entropy (defined in App.~\ref{sec:hyper}) that detects underdetermination whether seeds choose a global rule or resolve positions independently. Definitions of the named statistical tests are given in Appendix~\ref{sec:hyper}.

\subsection{Experimental setup}
\label{sec:setup}

All learners are small decoder-only transformers trained from scratch with next-token loss restricted to output positions (AdamW, cosine schedule, gradient clipping; early stopping at 99.9\% training accuracy). A ``seed'' denotes one training run from a distinct random initialization (and batch ordering); \emph{the training data is held fixed across seeds}, generated once from a fixed RNG, so seeds vary only in initialization and batch order. The counterfactual probe set is likewise fixed, aligning per-position decisions across seeds. The model's \emph{inductive bias}---the preferences introduced by architecture (e.g.\ the position-encoding scheme) and optimizer, rather than by the data---is therefore the only thing that can break the tie in the shadow condition; the architecture sweep makes this concrete by changing only the position encoding. Full hyperparameters are in Appendix~\ref{sec:hyper}. Code, generators, and analysis scripts are available at an anonymized repository.

All conditions use $N{=}30$ seeds (30 runs from different initializations) trained to the $99.9\%$ accuracy criterion; every reported model reached it. Statistics are computed over the 30 per-seed scores per condition.

\section{Type III: Basin Escape Mapping}
\label{sec:bem}

Type~III holds that a function can be \emph{representable} by an architecture yet \emph{unreachable} by standard training. Basin Escape Mapping (BEM) separates the two by construction; its results appear with the others in \S\ref{sec:results-bem}.

\paragraph{Design.}
The target is $k$-sparse parity on $n$ bits. Representability is \emph{proved}, not assumed: we hand-construct exact weights for a two-layer ReLU MLP computing the parity (App.~\ref{sec:bem-construction}), verified at $100\%$ accuracy for every $k$. Reachability is then \emph{measured} as the fraction of random-initialization runs that reach the function under (a) vanilla Adam, against three interventions: (b) initialization near the constructed solution plus Gaussian noise, (c) a curriculum increasing $k$ within the budget, and (d) width scaling. Training uses fresh online batches every step, so failure cannot be attributed to finite data. The reachable control is $k{=}1$, identical pipeline. We use $n{=}80$, $k\in\{1,\dots,10\}$, width $512$, $20$k steps, $N{=}25$ seeds (25 runs from different random initializations).

\section{Results}
\label{sec:results}

We report the three probes in taxonomy order (Type~I, II, III). In every case the shadow condition shows the predicted signature, and a dedicated control rules out a capacity or modality artifact; all effects are large and survive multiplicity correction.

\subsection{Type~I: text sits at a computable ceiling}
\label{sec:results-lcr}

Language Compression Residuals (design in \S\ref{sec:lcr}) compare a \emph{text learner}, which sees only a coarse $B$-bin quantization of a latent $z$, against a \emph{full-signal learner}, which sees $z$ directly. Table~\ref{tab:lcr} reports the result. On the fine target the full-signal learner reaches $0.996$, while the text learner plateaus at $0.741$--$0.746$---at the computed Bayes ceiling of $0.740$ (the small excess is within $N{=}8$ seed noise). The residual gap is $\approx0.25$ and \textbf{flat across two-and-a-half orders of magnitude} of data; a regression of gap on $\log$-budget gives slope $-4\times10^{-4}$ ($p=0.78$), i.e.\ no closure with scale, and the one-sided test that the gap exceeds zero at the largest budget gives $p=4.7\times10^{-4}$. The \textbf{coarse control}---a target that is a deterministic function of the text---\emph{inverts} the gap ($-0.11$: the text learner is now \emph{better}, reaching $1.000$), confirming the fine-task deficit is inexpressibility, not a generic modality disadvantage. The text learner is not undertrained: it is already \emph{optimal} given what the channel can express.

\begin{table}[t]
\centering
\small
\begin{tabular}{@{}lcccc@{}}
\toprule
\textbf{Task} & \textbf{Budget} & \textbf{Text} & \textbf{Full} & \textbf{Gap} \\
\midrule
Fine   & 300     & 0.741 & 0.990 & $+0.249$ \\
Fine   & 3{,}000  & 0.743 & 0.994 & $+0.251$ \\
Fine   & 30{,}000 & 0.744 & 0.994 & $+0.250$ \\
Fine   & 100{,}000 & 0.746 & 0.996 & $+0.250$ \\
\midrule
Coarse & 100{,}000 & 1.000 & 0.894 & $-0.106$ \\
\bottomrule
\end{tabular}
\caption{Language Compression Residuals ($N{=}8$ seeds). Fine-target text accuracy plateaus at the computable Bayes ceiling ($0.740$); the full-vs-text gap is flat in budget (slope $-4\times10^{-4}$, $p=0.78$; one-sided gap $p=4.7\times10^{-4}$). The coarse control reverses the gap, isolating inexpressibility from modality.}
\label{tab:lcr}
\end{table}

\begin{figure}[t]
  \centering
  \includegraphics[width=\columnwidth]{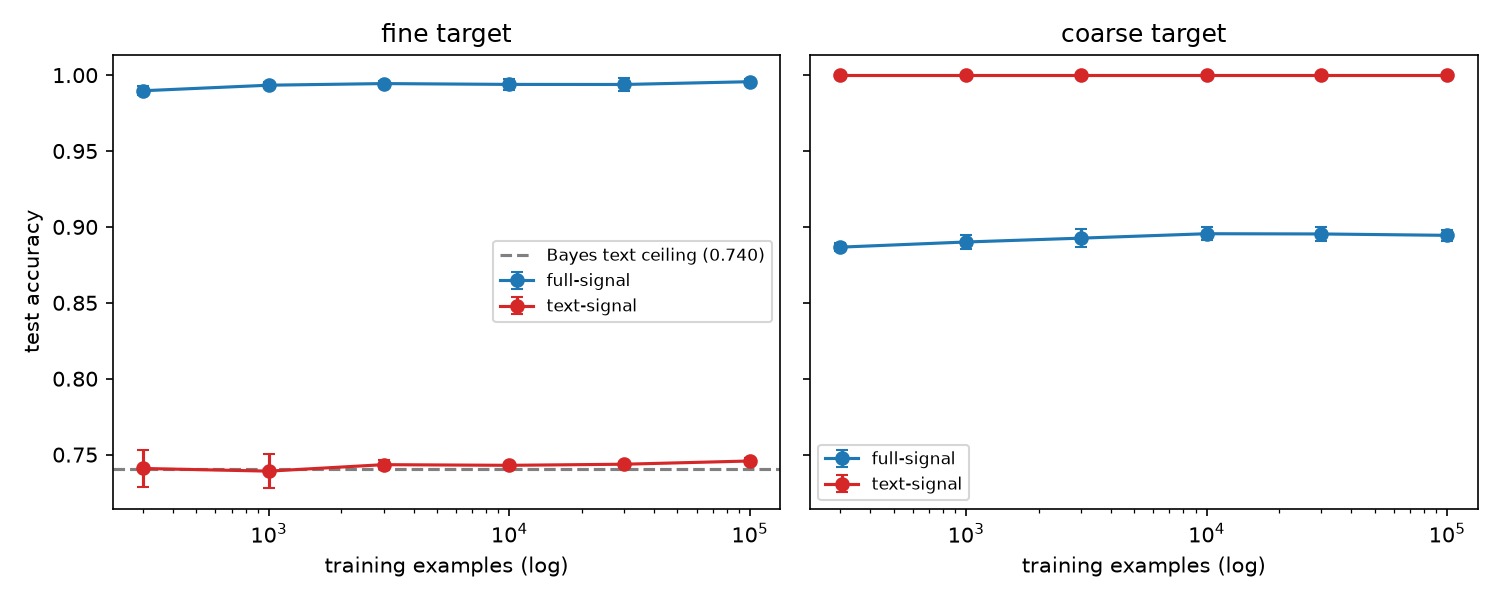}
  \caption{Scaling curves for the fine (left) and coarse (right) targets: full-signal vs.\ text-only accuracy, with the computed Bayes text ceiling (dashed). The fine-target gap is flat across budgets; the coarse gap is absent (text wins).}
  \label{fig:lcr}
\end{figure}

\subsection{Type~II: the identifiability dial}
\label{sec:results-pal}

\paragraph{Palindrome task.}
Table~\ref{tab:cdt}a reports the provable string task. Every condition reaches $\geq 99.7\%$ training accuracy, so capacity and convergence are matched; only identifiability differs. In the \textbf{shadow} condition, models settle at copy score $0.672$ (95\% CI $[0.660, 0.683]$): with the data consistent with both rules, behavior is pulled to an architecture-dependent operating point rather than committing to either rule. The two \textbf{controls} move it to the poles---$5\%$ \textsc{copy}-labeled data gives $0.976$, $5\%$ \textsc{reverse}-labeled data gives $0.102$---a bidirectional, maximal shift: Mann--Whitney $U$ tests give $p=3.0\times10^{-11}$ against each control, rank-biserial $r=+1.0$ and $r=-1.0$, Holm-adjusted $p=6.0\times10^{-11}$.

\paragraph{Dose--response.}
The contrast is not a threshold artifact of the chosen $5\%$. Sweeping $\rho \in \{0, 0.01, 0.02, 0.05, 0.10\}$ yields a smooth, monotone curve: copy score rises $0.672 \to 0.850 \to 0.924 \to 0.976 \to 0.992$, with per-position disagreement entropy falling in lockstep ($0.368 \to 0.025$ bits). Identifiability is a dial, not a switch.

\paragraph{Inductive bias sets the silent default.}
Replacing learned absolute with rotary position encodings shifts the data-silent operating point from $0.672$ to $0.931$ (Mann--Whitney $p=3.0\times10^{-11}$; distributions disjoint). The data is identical and equally non-identifiable in both, so the difference is attributable entirely to inductive bias---the predicted mechanism, made visible by varying only the bias.

\paragraph{Bridge task.}
Table~\ref{tab:cdt}b reports the agreement-attraction task. Every model reaches $100\%$ training accuracy \emph{and} $100\%$ rule consistency, the cleanest possible matched-capacity setting. The shadow condition adopts the \textsc{hierarchical} rule almost universally (subject score $0.997$), reproducing the known hierarchical bias of neural sequence models \citep{mccoy-etal-2020-syntax}---but here we can prove the data did not cause it, since the distribution is exactly consistent with \textsc{linear} too. Planting $5\%$ \textsc{hierarchical} sentences leaves behavior unchanged ($1.000$); planting $5\%$ \textsc{linear} sentences flips it completely ($0.000$, $p=3.0\times10^{-11}$, $r=-1.0$). The asymmetry is informative: signal that \emph{opposes} the bias overrides it, while signal that \emph{agrees} is redundant (Control-\textsc{hier}, $p=0.33$, intentionally not separable). Appendix~\ref{sec:comparison-table} positions these numbers against the closest prior measurements; the novel elements are provable ambiguity (Prop.~\ref{prop:nonid}), bidirectional steering, and the dose--response curve.

\begin{table}[t]
\centering
\small
\begin{tabular}{@{}lccc@{}}
\toprule
\multicolumn{4}{@{}l}{\emph{(a) Palindrome ($N{=}30$ seeds): copy score $\to1$ \textsc{copy}, $\to0$ \textsc{rev}}} \\
\textbf{Condition} & \textbf{Consist.} & \textbf{Copy (95\% CI)} & \textbf{Pref.} $s$ \\
\midrule
Shadow                & 0.83 & $0.672$ $[.660,.683]$ & $+0.34$ \\
Control-\textsc{copy} & 0.97 & $0.976$ $[.973,.979]$ & $+0.95$ \\
Control-\textsc{rev}  & 0.86 & $0.102$ $[.075,.133]$ & $-0.80$ \\
\midrule
\multicolumn{4}{@{}l}{\emph{(b) Bridge ($N{=}30$ seeds): subject score $1$ \textsc{hier}, $0$ \textsc{lin}}} \\
\textbf{Condition} & \textbf{Train} & \textbf{Subj.\ score} & \textbf{Pref.} $s$ \\
\midrule
Shadow                & 1.000 & $0.997$ & $+0.99$ \\
Control-\textsc{hier} & 1.000 & $1.000$ & $+1.00$ \\
Control-\textsc{lin}  & 1.000 & $\mathbf{0.000}$ & $\mathbf{-1.00}$ \\
\bottomrule
\end{tabular}
\caption{CDT results. \textbf{(a)} Palindrome: $5\%$ disambiguating data pins the rule at \emph{either} pole; Mann--Whitney $p=3.0\times10^{-11}$ vs.\ each control, $r=\pm1.0$, Holm-adjusted $p=6.0\times10^{-11}$; all conditions $\geq99.7\%$ train acc. \textbf{(b)} Bridge: same effect in agreement attraction; shadow vs.\ Control-\textsc{lin} $p=3.0\times10^{-11}$, $r=-1.0$. Shadow vs.\ Control-\textsc{hier} is intentionally \emph{not} separable ($p=0.33$): the shadow already adopts the hierarchical rule.}
\label{tab:cdt}
\end{table}

\subsection{Type~III: representable but unreachable}
\label{sec:results-bem}

Table~\ref{tab:bem} reports Basin Escape Mapping (design in \S\ref{sec:bem}). Vanilla Adam reaches the parity with probability $1.0$ for $k\le5$ but collapses at the frontier: $0.04$ at $k{=}6$ and exactly $0.0$ for $k\in\{7,8,9,10\}$---despite the function being representable at $100\%$ throughout (proved by hand-construction, App.~\ref{sec:bem-construction}). Mean steps-to-reach grow explosively up to the cliff ($250\to520\to2{,}200\to10{,}810$ for $k{=}1,2,4,5$), the fingerprint of a vanishing gradient signal \citep{shalev-shwartz2017failures}. At the headline $k{=}9$, the shadow is maximal (reach $0.0$), yet \textbf{initialization near the constructed solution} recovers it in $290$ steps ($100\%$ reach; Fisher exact $p=1.6\times10^{-14}$ vs.\ vanilla) and a curriculum recovers it more slowly ($100\%$, $\sim$19k steps). Crucially, \textbf{scaling width does not help}: at $k{=}9$ both a $4\times$ narrower ($128$) and a $4\times$ wider ($2048$) network reach $0.0$, indistinguishable from vanilla. The basin exists and is trivially reachable from the right start; standard training never finds it, and more parameters do not change that.

\begin{table}[t]
\centering
\small
\begin{tabular}{@{}llcc@{}}
\toprule
\textbf{Condition} & \textbf{$k$} & \textbf{Reach rate (95\% CI)} & \textbf{Steps} \\
\midrule
Hand-built & 1--10 & $1.00$ \emph{(exact)} & 0 \\
\midrule
Vanilla & 5 & $1.00$ $[0.87,1.00]$ & 10{,}810 \\
Vanilla & 6 & $0.04$ $[0.01,0.20]$ & 18{,}250 \\
Vanilla & 7--10 & $\mathbf{0.00}$ $[0.00,0.13]$ & --- \\
\midrule
Near-init & 9 & $\mathbf{1.00}$ $[0.87,1.00]$ & 290 \\
Curriculum & 9 & $1.00$ $[0.87,1.00]$ & 18{,}910 \\
Width 128/2048 & 9 & $0.00$ $[0.00,0.13]$ & --- \\
\bottomrule
\end{tabular}
\caption{Basin Escape Mapping ($n{=}80$, $N{=}25$ seeds). Representable at $100\%$ for all $k$, but vanilla Adam reaches $0\%$ for $k\ge7$. At $k{=}9$, near-initialization recovers $100\%$ in $290$ steps (Fisher $p=1.6\times10^{-14}$) while $4\times$ width scaling in either direction stays at $0\%$. Shadow magnitude (near-init $-$ vanilla) $=1.0$.}
\label{tab:bem}
\end{table}

\section{Discussion}
\label{sec:discussion}

\paragraph{All three signatures hold.}
Each shadow type was given a proof of its premise (a computable expressibility ceiling; provable non-identifiability; constructive representability), a shadow condition with a distinct measured signature (text at the $0.740$ ceiling; copy $0.672$; vanilla reach $0.00$), and a control that rules out a capacity or modality artifact. All effects are large (flat $0.25$ gap; $r=\pm1.0$; magnitude $1.0$) and corrected for multiplicity, and each resists the obvious ``more scale'' remedy by a dedicated control (flat slope $p=0.78$; dose--response; width $128/2048{=}0$).

\paragraph{Conjectures.}
The probes motivate three open conjectures: (\textbf{C1}) for any scale and architecture there exist Type~II functions non-identifiable from text but identifiable given interventional or non-linguistic signal; (\textbf{C2}) there exist representable functions whose basins are so rarely visited under SGD that scale does not raise acquisition probability (BEM's flat $0\%$ across a $16\times$ width range is consistent); and (\textbf{C3}) reducing Type~I shadows via multimodality may enlarge Type~III shadows by fragmenting the landscape.

\paragraph{Implications.}
\textbf{Benchmarks:} held-out accuracy cannot distinguish a learned rule from a reachable shortcut that agrees with it on the test distribution \citep{geirhos2020shortcut,mccoy-etal-2019-right}, suggesting counterfactual probes as a useful complement to capability claims. \textbf{Auditing:} a claimed capability should be verified for both representability and reachability, which BEM shows can diverge completely. \textbf{Safety:} in shadow-affected regions a model need not be right for the right reasons, so confident outputs there should be treated as hypotheses, and shadow-aware uncertainty is a distinct capability worth training for.

\section{Conclusion}

The information shadow reframes a family of language-model limitations as structural rather than incidental, and makes each kind measurable. We validate all three predicted signatures in controlled settings where the premise of each is provable: inexpressibility, where a text learner sits at a computable ceiling while a full-signal learner pulls away by a gap that does not close with scale; non-identifiability, where silent data cedes the rule to inductive bias and $5\%$ of disambiguating data re-pins it in either direction; and unreachability, where a function that is representable at $100\%$ is reached $0\%$ of the time by standard training yet instantly from the right initialization, with width scaling no help. Mapping these shadows in large pretrained models, and turning the probes into auditing tools, is the natural next step.

\section*{Limitations}

Our experiments use small models on synthetic tasks by design: isolating each shadow type requires a setting where its premise is provable, which large naturalistic corpora cannot offer. The three constructions are by design maximally favorable---each premise holds as a theorem---so the quantitative separations bound the phenomena in idealized settings and are not direct estimates of their magnitude in trained LMs; generalization to naturalistic learning is a conjecture the probes are meant to test, not a result we claim. Consequently, the quantitative numbers characterize the \emph{phenomena}, not their extent in frontier models; measuring shadow magnitude in pretrained LMs is future work. The three conjectures remain open, the Type~I/Type~II boundary may be hard to adjudicate for real-world phenomena, and the taxonomy does not claim exhaustiveness.

\bibliography{custom}

\appendix

\section{Hyperparameters, Statistical Tests, and Reproducibility}
\label{sec:hyper}

\paragraph{Named statistical tests.}
We use the following standard tests; definitions are collected here for reference.
\emph{Mann--Whitney $U$ test:} a non-parametric two-sample test of whether one distribution is stochastically shifted relative to another, based on rank sums rather than means; it makes no normality assumption, suiting our bounded scores.
\emph{Rank-biserial correlation $r$:} an effect size for the Mann--Whitney test, equal to the difference between the proportion of cross-group pairs favoring each group; $r=+1$ ($-1$) means every sample in one group exceeds (falls below) every sample in the other, i.e.\ the two distributions are completely separated.
\emph{Holm--Bonferroni correction:} a sequential multiple-comparison adjustment that controls the family-wise error rate; it orders the raw $p$-values and applies progressively less stringent thresholds, more powerful than plain Bonferroni while equally conservative on the smallest $p$.
\emph{Levene's test:} a test for equality of variance between groups, used here to check whether conditions differ in spread and not only in central tendency.
\emph{Fisher's exact test} (used for BEM): an exact test of association in a $2\times2$ table of counts (here, reached vs.\ not-reached across conditions), appropriate when reach rates are at the $0$/$1$ boundary.

\paragraph{Main-run configuration.}
All CDT learners are decoder-only transformers (3 layers, 4 heads, $d_{\text{model}}{=}128$, $d_{\text{ff}}{=}512$; learned absolute position encodings, plus rotary in the architecture sweep). The palindrome task uses a 20-symbol vocabulary ($+4$ special), sequence lengths $8$--$16$, $50{,}000$ training pairs, and a $2{,}000$-input probe set. Training: AdamW (lr $3{\times}10^{-4}$, weight decay $0.01$), cosine schedule with $200$ warmup steps, batch size $256$, up to $4{,}000$ steps with early stopping at train accuracy $\geq 0.999$, $30$ seeds per condition. Training data is generated from a fixed RNG and shared across seeds, so seeds vary only initialization and batch order; the probe set is fixed across all models. The bridge task uses 4 noun lemmas, up to 3 modifiers, $40{,}000$ training sentences, and a $2{,}000$-sentence counterfactual set; all other settings match. Decision vectors (per-position rule choices on the fixed probe set) are saved per seed, enabling exact recomputation of all statistics and figures from released artifacts.

\paragraph{Disagreement entropy.}
The across-seed per-position disagreement entropy is the mean, over counterfactual positions, of the binary entropy of the seeds' \textsc{copy}/\textsc{rev} choices at each position. It is $0$ when all seeds agree at every position and approaches $1$ bit when seeds split evenly, capturing underdetermination whether seeds adopt a global rule or resolve positions independently.

\section{Comparison with Prior Findings}
\label{sec:comparison-table}

\begin{table}[h]
\centering
\small
\begin{tabular}{@{}p{1.9cm}p{2.0cm}p{2.6cm}@{}}
\toprule
\textbf{Study} & \textbf{Prior finding} & \textbf{This work} \\
\midrule
\citet{mccoy-etal-2020-berts}: 100 BERT seeds, MNLI$\to$HANS & OOD acc.\ 0.0--66.2\% despite in-dist.\ 83.6--84.8\% & shadow copy $0.672{\pm}0.03$; controls $0.976$/$0.102$ \\
\addlinespace
\citet{warstadt-etal-2020-learning}: MSGS & MCC preference; inoculation shifts toward linguistic rule & $s$: shadow $+0.34$; controls $\pm$ \emph{bidirectional}; dose--response \\
\addlinespace
\citet{mccoy-etal-2020-syntax}: seq2seq & \% hierarchical tracks architecture & bias-tracked: $0.672$ (abs) vs.\ $0.931$ (rotary) \\
\bottomrule
\end{tabular}
\caption{Relation to prior measurements of underdetermination. Prior ambiguity is empirical; ours is provable (Prop.~\ref{prop:nonid}). Bidirectional steering and the dose--response curve are, to our knowledge, new.}
\label{tab:comparison}
\end{table}

\section{BEM Construction (Representability Proof)}
\label{sec:bem-construction}

For $k$-sparse parity, let $s=\sum_{i=1}^{k} x_i$ be the sum of the relevant bits, an integer in $[0,k]$. For $j=1,\dots,k$ define the ramp $R_j(s)=\mathrm{ReLU}(s-(j{-}1))-\mathrm{ReLU}(s-j)$, which equals $\mathbf{1}[s\ge j]$ at integer $s$. Then $\mathrm{parity}(s)=\sum_{j=1}^{k}(-1)^{j+1}R_j(s)$ equals $s \bmod 2$ at every integer. We realize each $R_j$ with two hidden ReLU units sharing the input row $w=\mathbf{1}[\text{relevant bits}]$ and biases $-(j{-}1)$ and $-j$, with signed output weights, and scale the output to a confident logit $C(2\,\mathrm{parity}-1)$. This requires width $\ge 2k$; unused units are zeroed. The construction is exact at all binary inputs, giving $100\%$ accuracy independent of training (verified empirically for all $k\in\{1,\dots,10\}$).

\section{A2/A3 Configuration}
\label{sec:a2a3-config}

\textbf{BEM (A2):} two-layer ReLU MLP; $n{=}80$ bits; $k\in\{1,\dots,10\}$; width $512$ (sweep $128$, $2048$); Adam, lr $10^{-3}$; online batches of $512$; $20$k steps; reach threshold $99\%$ accuracy on $8192$ held-out samples; near-init noise $\sigma{=}0.05$; curriculum $k:2\!\to\!k$; $N{=}25$ seeds. \textbf{LCR (A3):} two-hidden-layer MLP (width $64$); $B{=}8$ bins; close-pair gap $g{=}0.12$; budgets $\{300,10^3,3{\times}10^3,10^4,3{\times}10^4,10^5\}$; Adam, lr $10^{-3}$; $6$k steps; fixed $20$k-sample test set; $N{=}8$ seeds. Both runners are fully resumable. All $p$-values are two-sided unless noted; BEM uses Fisher's exact test, LCR uses Mann--Whitney (one-sided for the directional gap) and OLS for the slope.

\end{document}